\documentclass[11pt]{article}

\usepackage{epsfig,epsf,fancybox}
\usepackage{amsmath}
\usepackage{amsthm}
\usepackage{mathrsfs}
\usepackage{amssymb}
\usepackage{graphicx}
\usepackage{color}
\usepackage{multirow}
\usepackage{paralist}
\usepackage{verbatim}
\usepackage{galois}
\usepackage{algorithm}
\usepackage{algorithmic}
\usepackage{boxedminipage}
\usepackage{booktabs}
\usepackage{accents}
\usepackage{stmaryrd}
\usepackage{subfig}

\newtheorem{theorem}{Theorem}

\newtheorem{lemma}[theorem]{Lemma}

\newtheorem{remark}[theorem]{Remark}
\newtheorem{assumption}[theorem]{Assumption}
\usepackage{thmtools,thm-restate}

\usepackage{bm}
\usepackage{mwe}

\newcommand{\be}{\begin{equation}}
\newcommand{\ee}{\end{equation}}
\newcommand{\ba}{\begin{array}}
\newcommand{\ea}{\end{array}}

\newcommand{\RR}{\mathbf R}

\newcommand{\ex}{{\bf\mathbb{E}}}

\newcommand{\XCal}{\mathcal{X}}
\newcommand{\YCal}{\mathcal{Y}}
\newcommand{\etal}{\textit{{et al. }}}

\newcommand{\st}{\textnormal{s.t.}}

\newcommand{\Prob}{\textnormal{Prob}}

\,
\,

\newcommand{\argmin}{\mathop{\rm argmin}}
\newcommand{\argmax}{\mathop{\rm argmax}}

\newcommand{\bpm}{\begin{pmatrix}}
\newcommand{\epm}{\end{pmatrix}}

\title{On the Iteration Complexity Analysis of Stochastic Primal-Dual Hybrid Gradient Approach with High Probability}

\author{
Linbo Qiao
\thanks{College of Computer, National University of Defense Technology, Changsha, China.  Email: linboqiao@gmail.com} 
\and
Tianyi Lin
\thanks{Department of Industrial Engineering and Operations Research, UC Berkeley, California, USA.}
\and
Qi Qin
\thanks{Department of Statistics, Stanford University, USA.} 
\and
Xicheng Lu$^*$ 
}

\begin{document}
\maketitle

\begin{abstract}
In this paper, we propose a stochastic Primal-Dual Hybrid Gradient (PDHG) approach for solving a wide spectrum of regularized stochastic minimization problems, where the regularization term is composite with a linear function. It has been recognized that solving this kind of problem is challenging since the closed-form solution of the proximal mapping associated with the regularization term is not available due to the imposed linear composition, and the per-iteration cost of computing the full gradient of the expected objective function is extremely high when the number of input data samples is considerably large. 

Our new approach overcomes these issues by exploring the special structure of the regularization term and sampling a few data points at each iteration. Rather than analyzing the convergence in expectation, we provide the detailed iteration complexity analysis for the cases of both uniformly and non-uniformly averaged iterates with high probability. This strongly supports the good practical performance of the proposed approach. Numerical experiments demonstrate that the efficiency of stochastic PDHG, which outperforms other competing algorithms, as expected by the high-probability convergence analysis.
\end{abstract}


\section{Introduction}\label{sec:introduction}
In this paper, we are interested in solving a class of compositely regularized convex optimization problems:
\be\label{prob}
\displaystyle \min_{x\in\XCal} \ \ex\left[l(x,\xi) \right] + r(Fx),
\ee
where $x\in\RR^d$,  
$\XCal$ is a convex compact set with diameter $D_x$, 
$r:\RR^l\rightarrow\RR$ is a convex regularization function, 
and $F\in\RR^{l\times d}$ is a penalty matrix (not necessarily diagonal) specifying the desired structured sparsity pattern in $x$. 
Furthermore, we denote $l(\cdot,\cdot): \RR^d\times\Omega\rightarrow\RR$ as a smooth convex function when applying a prediction rule $x$ on a sample dataset $\{ \xi_i=(a_i,b_i) \}$, and the corresponding expectation is denoted by $l(x) = \ex\left[l(x,\xi) \right]$.

When $F = I$, the above formulation accommodates quite a few classic classification and regression models including 
Lasso obtained by setting $l(x,\xi_i) = \frac{1}{2}\left\|a_i^\top x - b_i\right\|^2$ and $r(x) = \lambda\left\|x\right\|_1$, 
and linear support vector machine (SVM) obtained by letting $l(x,\xi_i) = \max\left(0,1-b_i \cdot a_i^\top x\right)$ and $r(x) = (\lambda/2)\left\|x\right\|_2^2$, where $\lambda>0$ is a parameter.
Moreover, the general structure of $F$ enables problem~\eqref{prob} to cover more complicated problems arising from machine learning, such as the fused Lasso \cite{Tibshirani-2011-Solution}, fused logistic regression and graph-guided regularized minimization \cite{Friedman-2009-Elements}.

However, this modeling power also comes with a challenge in computation. In particular, when $F$ is not diagonal, it is very likely that the proximal mapping associated with $r(Fx)$ does not admit a closed-form expression. To cope with this difficulty, we could reformulate problem~\eqref{prob} as a convex-concave saddle point problem by exploiting some special structure of the regularization term, and then resort to the Primal-Dual Hybrid Gradient (PDHG) approach \cite{Zhu-2008-Efficient}. This approach has exhibited attractive numerical performance in image processing and image restoration applications \cite{Esser-2010-General,Bonettini-2012-Convergence,Zhu-2008-Efficient,Zhang-2011-Unified}.
We refer readers to \cite{Chambolle-2011-First,Goldstein-2013-Adaptive,He-2012-Convergence} to visit convergence properties of PDHG and its variants.

In practice, $\ex\left[l(x,\xi) \right]$ is often replaced by its empirical average on a set of training samples. 
In this case, the computational complexity of calling the function value or the full gradient of $l(x)$ is proportional to the number of training samples, which is extremely huge for modern data-intensive applications. This could make PDHG and linearized PDHG suffer severely from the very poor scalability. Therefore, it is promising to propose a Stochastic variant of PDHG (SPDHG). Like many well-studied incremental or stochastic gradient methods \cite{Xiao-2010-Dual,Schmidt-2013-Minimizing,Lan-2015-optimal,Agarwal-2012-Information,Zhang-2015-Stochastic}, we draw a sample $\xi^{k+1}$ in random and compute a noisy gradient $\nabla l(x^k,\xi^{k+1})$ at the $k$-th iteration with the current iterate $x^k$. As a result, the proposed SPDHG method enjoys the capability of dealing with very large-scale datasets.

Another way to handle the non-diagonal $F$ and the expected objective function $\ex\left[l(x,\xi) \right]$ is stochastic ADMM-like methods \cite{Ouyang-2013-Stochastic,Zhao-2015-Adaptive,Gao-2014-Information,Suzuki-2013-Dual,Wang-2012-Online,Azadi-2014-Towards,Zhong-2014-Fast,Suzuki-2014-Stochastic} which aim for solving the following problem after introducing an additional variable $z$:
\be\label{prob-ADMM}
\displaystyle \min_{x\in\XCal, z = Fx} \ l(x) + r(z),
\ee
whose augmented Lagrangian function is given by $l(x) + r(z) + \lambda^{\top}(z - Fx) + \frac{\gamma}{2}\| z - Fx \|^2_2$. Comparing this function with the convex-concave problem~\eqref{prob} in Section \ref{sec:background}, we can see that ADMM-like methods need to update one more vector variable than PDHG-type methods in every iteration. Thus, it can be expected that the per-iteration computational cost of ADMM-like methods is higher than our proposed algorithm SPDHG, as confirmed by the numerical experiments in Section \ref{sec:experiment}.

\textbf{Our contribution:} To the best of our knowledge, we propose in this paper a new convex-concave formulation of problem~\eqref{prob}, { as well as the first stochastic variant of the PDHG algorithm for both uniformly and non-uniformly averaged iterates with achievable iteration complexities. We also present the iteration complexity analysis of the proposed algorithms in the sense of high-probability.} In particular, for the proposed algorithm, the probability of converging with a rate higher than $O(\frac{1}{\sqrt{t}})$ can be exponentially small, which has also been established for the well-known stochastic ADMM (SADMM) \cite{Ouyang-2013-Stochastic}. Moreover, as mentioned before, the significant advantage gained by SPDHG beyond SADMM is the low per-iteration complexity. The effectiveness and efficiency of the proposed SPDHG algorithm are demonstrated by encouraging empirical evaluation in graph-guided regularized minimization tasks on different genres of real-world datasets, the convergence { behaviors} strongly support the high-probability analysis. 

\textbf{Related work:} Given the importance of problem~\eqref{prob}, various stochastic optimization algorithms have been proposed to solve problem~\eqref{prob} or the more general form of problem~\eqref{prob}, which can be written { as}:
\begin{eqnarray}\label{prob-general}
& \displaystyle \min_{x\in\XCal,y\in\RR^d} \ \ex\left[l(x,\xi) \right] + r(y), & \st \ Ax + By = b.
\end{eqnarray}
It is easy to verify that problem~\eqref{prob} is a special case of problem~\eqref{prob-general} when $A=F$, $B=-I$ and $b=0$. 

In solving problem~\eqref{prob-general}, Wang and Banerjee \cite{Wang-2012-Online} proposed an online ADMM that requires an easy proximal mapping of $l$. 
However, this is difficult for many loss functions such as logistic loss function. Ouyang \etal \cite{Ouyang-2013-Stochastic}, Suzuki \cite{Suzuki-2013-Dual}, Azadi and Sra \cite{Azadi-2014-Towards}, Gao \etal \cite{Gao-2014-Information}, and recently Zhao \etal \cite{Zhao-2015-Adaptive} developed several stochastic variants of ADMM, which linearize $l$ by using its noisy subgradient or gradient and add a varying proximal term.
Furthermore, Zhong and Kwok \cite{Zhong-2014-Fast} and Suzuki \cite{Suzuki-2014-Stochastic} respectively proposed a stochastic averaged gradient-based { ADMM} and a stochastic dual coordinate ascent { ADMM}, which can both obtain improved iteration complexities. However, these methods did not explore the structure of $r$ and need to update one more vector variable than PDHG-type methods in every iteration. We will show in the experiments that our proposed SPDHG algorithm is far more efficient than these algorithms.

It is worth mentioning that some other stochastic version of the primal-dual gradient approach was also analyzed in recent work \cite{Lan-2015-optimal, Chambolle-2017-Stochastic}. However, their convex-concave formulation is different from ours, and their algorithm cannot be applied to solve problem~\eqref{prob}. Regarding the iteration complexity, the proposed SPDHG algorithm has accomplished the best possible one for first-order stochastic algorithms under general convex objective functions \cite{Agarwal-2012-Information}. A better convergence rate of $O(\frac{1}{t^2}+\frac{1}{\sqrt{t}})$ can be obtained by using Nesterov's acceleration technique in \cite{Xiao-2010-Dual}. 

The most related algorithm to our proposed SPDHG algorithm { are} the SPDC algorithm \cite{Zhang-2015-Stochastic} { and} its adaptive variant \cite{Zhu-2015-Adaptive}. Similar to our SPDHG algorithm, the SPDC algorithm is also a stochastic variant of the batch primal-dual algorithm developed by Chambolle and Pock \cite{Chambolle-2011-First}, which alternates between maximizing over a randomly chosen dual variable and minimizing over the primal variable. However, the SPDC algorithm does not explore the special structure of the regularization term (Assumption \ref{assumption_3}), and their convex-concave formulation is different from ours. This leads to the inability of the SPDC algorithm to solve problem~\eqref{prob}. Specifically, \cite{Zhang-2015-Stochastic} suggests to reformulate problem~\eqref{prob} as
\begin{equation}\label{prob-dual}
\displaystyle  \min_{x\in\XCal}\max_{y\in\RR^d} \left\{ \ex\left[ \left\langle y, x\right\rangle - l^*(y,\xi)\right] + r(Fx)\right\},
\end{equation}
where $l^*(y,\xi) = \sup_{\alpha\in\RR^d}\left\{\left\langle\alpha, y\right\rangle - l(\alpha,\xi)\right\}$ is the convex conjugate of $l(x,\xi)$. Then the SPDC algorithm in solving problem~\eqref{prob-dual} requires that the proximal mapping of $l^*$ and $r(Fx)$ { are} easily computed, which is somewhat strong for a variety of application problems. In addition, the SPDC algorithm requires $r$ to be strongly convex.

In contrast, our SPDHG algorithm only needs the smoothness of $l$ and the convexity of $r$, and hence efficiently solves a wide range of graph-guided regularized optimization problems, which cannot be solved by the SPDC algorithm and its adaptive variant.

\section{Problem Set-Ups}\label{sec:background}
\subsection{Assumptions}
We make the following assumptions (Assumption \ref{assumption_1}-\ref{assumption_4}) regarding problem~\eqref{prob} throughout the paper:
\begin{assumption}\label{assumption_1}
The optimal set of problem~\eqref{prob} is nonempty.
\end{assumption}

\begin{assumption}\label{assumption_2}
$l(\cdot)$ is continuously differentiable with Lipschitz continuous gradient. That is, there exists a constant $L>0$ such that
\begin{displaymath}
\left\|\nabla l(x_1) - \nabla l(x_2)\right\|\leq L\left\| x_1 - x_2\right\|, \forall x_1, x_2\in\XCal.
\end{displaymath}
\end{assumption}
Many formulations in machine learning satisfy Assumption \ref{assumption_2}. The following least square and logistic functions are two commonly used ones:
\begin{displaymath}
l(x,\xi_i) = \frac{1}{2}\left\| a_i^\top x - b_i\right\|^2 \mbox{ or }\\ \ l(x,\xi_i) = \log\left( 1+\exp\left(-b_i \cdot a_i^\top x\right) \right),
\end{displaymath}
where $\xi_i = \left(a_i,b_i\right)$ is a single { sample point}.

\begin{assumption}\label{assumption_3}
$r(x)$ is a continuous function which is possibly non-smooth, and it can be described as follows:
\begin{displaymath}
r(x) = \max_{y\in\YCal}\left\langle y,x\right\rangle,
\end{displaymath}
where $\YCal\in\RR^d$ is a convex compact set with diameter $D_y$.
\end{assumption}
Note that Assumption \ref{assumption_3} is reasonable for the learning problems with a norm regularization such as $\ell_1$-norm or nuclear norm:
\begin{eqnarray*}
& \left\| x\right\|_1 & = \max\left\{ \left\langle y,x\right\rangle | \left\| y\right\|_{\infty} \leq 1\right\}, \\
& \left\| X\right\|_* & = \max\left\{ \left\langle Y,X\right\rangle | \left\| Y\right\|_{2} \leq 1\right\}. 
\end{eqnarray*}
\begin{assumption}\label{assumption_4}
The function $l(x)$ is easy for gradient estimation. 
That is to say, any stochastic gradient estimation $\nabla l(\cdot,\xi)$ for $\nabla l(\cdot)$ at x satisfies
\begin{displaymath}
\ex\left[\nabla l(x,\xi)\right] = \nabla l(x),
\end{displaymath}
and
\begin{displaymath}
\ex\left[\left\| \nabla l(x,\xi) - \nabla l(x)\right\|^2\right] \leq \sigma^2, 
\end{displaymath}
and
\begin{displaymath}
\ex\left[\exp\left(\frac{\left\|\nabla l(x,\xi) - \nabla l(x)\right\|^2}{\sigma^2}\right)\right] \leq \exp\left(1\right), 
\end{displaymath}
\end{assumption}
where $\sigma>0$ is a constant. { Indeed, $\sigma^2$ can be interpreted as the variance of stochastic gradient estimation $\nabla l(x,\xi)$, and need to be finite in the iteration complexity analysis of stochastic gradient methods. We refer the interested readers to \cite{Nemirovski-1983-Problem,Nemirovski-2009-Robust,Lan-2012-Optimal} for more details. }

\begin{assumption}\label{assumption_5}
$l(\cdot)$ is \textbf{$\mu$-strongly convex} at $x$. In other words, there exists a constant $\mu>0$ such that
\begin{displaymath}
l(y) - l(x) - \left( y - x\right)^\top\nabla l(x) \geq \frac{\mu}{2}\left\| y - x\right\|^2, \quad \forall y\in\XCal.
\end{displaymath}
\end{assumption}
We remark that Assumption \ref{assumption_5} is optional, and it is only necessary for the theoretical analysis that can lead to a lower iteration complexity.

\subsection{Convex-Concave Saddle Point Problem}
According to Assumption \ref{assumption_3}, we are able to rewrite problem~\eqref{prob} as the following convex-concave saddle point problem:
\be\label{prob:N}
\min_{x\in\XCal}\max_{y\in\YCal}\left\{ P(y,x) = l(x) + \left\langle y,Fx\right\rangle\right\}.
\ee
\begin{remark}
We remark here that the formulation~\eqref{prob:N} is greatly different from those used in \cite{Lan-2015-optimal, Zhang-2015-Stochastic, Zhu-2015-Adaptive}, where they formulate problem~\eqref{prob} as another convex-concave saddle point problem~\eqref{prob-dual} by using the convex conjugate of $l(\cdot)$. Therefore, their algorithms are limited to solving problem~\eqref{prob} due to the fact that the proximal mapping of $r(Fx)$ is difficult to compute. 
\end{remark}
This problem can be solved by Linearized PDHG (LPDHG) with the following iteration scheme:
\begin{eqnarray}
y^{k+1} & := & \argmax_{y\in\YCal} \left\{ P(y,x^k) - \frac{1}{2s}\left\| y-y^k\right\|^2 \right\} \label{LPDHG_Update_Y}, \\
x^{k+1} & := & \prod_{\XCal}\left[ x^k - \beta\left( \nabla l(x^k) + F^\top y^{k+1}\right) \right] \label{LPDHG_Update_X}.
\end{eqnarray}
where $\prod_{\XCal}(x)$ is the projection of $x$ onto the convex set $\XCal$, { and $s>0$ is a positive constant. As mentioned before, the above algorithm is more efficient than ADMM-like methods. Indeed, the scheme of Gradient-based ADMM \cite{Gao-2014-Information} for solving problem~\eqref{prob-ADMM} is as follows,
\begin{eqnarray*}
z^{k+1} & := & \argmin \left\{ r(z) - \left\langle\lambda^k, z - Fx^k \right\rangle + \frac{\rho}{2}\left\| z-Fx^k \right\|^2 \right\}, \\
x^{k+1} & := & \prod_{\XCal} \left\{ x^k - \rho\left(\nabla l(x^k) + F^\top\lambda^k - \rho(z^{k+1}-Fx^k)\right)\right\},  \\
\lambda^{k+1} & := & \lambda^k - \rho\left(z^{k+1}-Fx^{k+1}\right), 
\end{eqnarray*}
where one more vector $\lambda$ needs to be updated.}

However, the above algorithm is inefficient since computing $\nabla l(x^k)$ in each iteration is very costly when the total number of samples $n$ is large. This inspires us to propose a stochastic variant of PDHG, where only the noisy gradient $\nabla l(x^k,\xi^{k+1})$ is computed at each step.

\section{The Algorithm and Convergence Results}\label{sec:SPDHG}

\begin{algorithm}[t]
\caption{SPDHG}\label{ALG:SPDHG}
\begin{algorithmic}
\STATE \textbf{Initialize:} $x^0$ and $y^0$.
\FOR {$k = 0,1,2,\cdots$}
\STATE Choose one data sample $\xi^{k+1}$ randomly.
\STATE Update $y^{k+1}$ according to Eq.~\eqref{LPDHG_Update_Y}.
\STATE Update $x^{k+1}$ according to Eq.~\eqref{SPDHG_Update_X}.
\ENDFOR
\STATE \textbf{Output:} $\bar{x}^t = \sum\limits_{k=0}^t \alpha^{k+1} x^{k+1}$ and $\bar{y}^t = \sum\limits_{k=0}^t \alpha^{k+1} y^{k+1}$.
\end{algorithmic}
\end{algorithm}

In this section, we first propose our Stochastic Primal-Dual Hybrid Gradient (SPDHG) algorithms with either uniformly or non-uniformly averaged iterates for solving problem~\eqref{prob:N}; and then provide the main result of the convergence property of the proposed algorithm.

\subsection{Algorithm}
The SPDHG is presented in Algorithm \ref{ALG:SPDHG}, where we have addressed the following three important issues: 
how to apply the noisy gradient, how to choose the step-size, and how to determine the weights for the non-uniformly averaged iterates.

\textbf{Stochastic Gradient:} 
Our SPDHG algorithm shares some common features with the LPDHG algorithm. 
In fact, the $y$-subproblems for both algorithms are essentially the same, while for the $x$-subproblem we adopt the noisy gradient $\nabla l(x^k,\xi^{k+1})$ in SPDHG rather than the full gradient $\nabla l(x^k)$ in LPDHG, \textit{i.e.},
\begin{equation}\label{SPDHG_Update_X}
x^{k+1} := \prod_{\XCal}\left[ x^k - \beta^{k+1}\left( \nabla l(x^k,\xi^{k+1}) + F^\top y^{k+1}\right) \right], 
\end{equation}
where $\prod_{\XCal}(x)$ is the projection of $x$ onto the convex set $\XCal$. That is, in SPDHG we first maximize over the dual variable and then perform one-step projected stochastic gradient descent along the direction $- \nabla l(x^k,\xi^{k+1}) - F^\top y^{k+1}$ with step-size $\beta^{k+1}$.

\textbf{The Step-Size $\beta^{k+1}$:} 
The choice of the step-size $\beta^{k+1}$ varies with respect to the different conditions satisfied by the objective function $l(\cdot)$. 
Different step-size rules also lead to different convergence rates.
Note that a sequence of vanishing step-sizes is necessary since we do not adopt any technique of variance reduction in the SPDHG algorithm.

\textbf{Non-uniformly Averaged Iterates:}
It was shown in \cite{Azadi-2014-Towards} that the non-uniformly averaged iterates generated by stochastic algorithms converge with fewer iterations. 
Inspired by their work, through non-uniformly averaging the iterates of the SPDHG algorithm and adopting a slightly modified step-size, we manage to establish an accelerated convergence rate of $O(\frac{1}{t})$ in expectation.

For the convenience of readers, we summarize the convergence properties with respect to different settings in Table \ref{tab:result}.
\begin{table}[t]
\caption{Convergence properties.}\label{tab:result}
\centering
\begin{tabular}{|c|c|c|c|} \hline
$l$ & General Convex & \multicolumn{2}{c|}{Strongly Convex} \\ \hline
$\beta^{k+1}$ & $\frac{1}{\sqrt{k+1}+L}$ & $\frac{1}{\mu (k+1)+L}$ & $\frac{2}{\mu (k+2)+2L}$ \\ \hline
$\alpha^{k+1}$ &  \multicolumn{2}{c|}{$\frac{1}{t+1}$} & $\frac{2(k+1)}{(t+1)(t+2)}$ \\ \hline
Rate & $O(\frac{1}{\sqrt{t}})$ & $O(\frac{\log(t)}{t})$ & $O(\frac{1}{t})$ \\ \hline
\end{tabular}
\end{table}

\subsection{Convergence of uniformly averaging under convex objective functions}
In this subsection, we present the convergence result for uniformly averaged iterates under general convex objective functions in the following theorem.
\begin{theorem}\label{Theorem-PDHG-Convergence}
Denote $\beta^{k+1}$ and $\alpha^{k+1}$ as shown in Table \ref{tab:result}, then for any stationary point $\left(y^*,x^*\right)$ of problem~\eqref{prob:N}, $\left(\bar{y}^{t},\bar{x}^{t}\right)$ generated by Algorithm~\ref{ALG:SPDHG} converges to $\left(y^*,x^*\right)$ with $O(\frac{1}{\sqrt{t}})$ rate in expectation.
\end{theorem}
We further present the high-probability result for uniformly averaged iterates under general convex objective functions in the following theorem.
\begin{theorem}\label{Theorem-PDHG-High-Probability}
Denote $\beta^{k+1}$ and $\alpha^{k+1}$ as shown in Table \ref{tab:result}, then for any stationary point $\left(y^*,x^*\right)$ of problem~\eqref{prob:N}, { $\left(\bar{y}^{t},\bar{x}^{t}\right)$ generated by Algorithm~\ref{ALG:SPDHG} converges to $\left(y^*,x^*\right)$ with $O(\frac{1}{\sqrt{t}})$ rate in the sense of high probability. More specifically, the following statement holds true,}
\begin{eqnarray*}
& & \Prob\left(P(y^*,\bar{x}^t) - P(\bar{y}^t,x^*) > \frac{D_y^2}{2s(t+1)} + \frac{L D_x^2}{2(t+1)} + \frac{D_x^2+2\lambda_{\max}(F^\top F)D_y^2}{\sqrt{t+1}} \right. \\
& & \left. \quad \quad  + \frac{2\sqrt{\Omega} D_x \sigma}{\sqrt{t+1}} + \frac{(1+\Omega) \sigma^2}{\sqrt{t+1}}\right)\leq 2\exp\left(-\Omega\right),
\end{eqnarray*}
for any $\Omega>0$. 
\end{theorem}

\subsection{Convergence of uniformly averaging under strongly convex objective functions}
In this subsection, we present the convergence result in the following theorem when the objective function is further assumed to be strongly convex.
\begin{theorem}\label{Theorem-Strong-PDHG-Convergence}
Denote $\beta^{k+1}$ and $\alpha^{k+1}$ as shown in Table \ref{tab:result}, then for any stationary point $\left(y^*,x^*\right)$ of problem~\eqref{prob:N}, $\left(\bar{y}^{t},\bar{x}^{t}\right)$ generated by Algorithm~\ref{ALG:SPDHG} converges to $\left(y^*,x^*\right)$ with $O(\frac{\log(t)}{t})$ rate in expectation.
\end{theorem}
We further present the high-probability result in the following theorem when the objective function is further assumed to be strongly convex.
\begin{theorem}\label{Theorem-Strong-PDHG-High-Probability}
Denote $\beta^{k+1}$ and $\alpha^{k+1}$ as shown in Table \ref{tab:result}, then for any stationary point $\left(y^*,x^*\right)$ of problem~\eqref{prob:N}, { $\left(\bar{y}^{t},\bar{x}^{t}\right)$ generated by Algorithm~\ref{ALG:SPDHG} converges to $\left(y^*,x^*\right)$ with $O(\frac{1}{\sqrt{t}}+\frac{\log(t)}{t})$ rate in the sense of high probability. More specifically, the following statement holds true,}
\begin{eqnarray*}
& & \Prob\left(P(y^*,\bar{x}^t) - P(\bar{y}^t,x^*) > \frac{D_y^2}{2s(t+1)} + \frac{L D_x^2}{2(t+1)} + \frac{\lambda_{\max}(F^\top F)D_y^2\log(t+1)}{\mu(t+1)} \right. \\
& & \left. \quad \quad  + \frac{2\sqrt{\Omega} D_x \sigma}{\sqrt{t+1}}  + \frac{(1+\Omega) \sigma^2 \log(t+1)}{2\mu(t+1)}\right) \leq 2\exp\left(-\Omega\right),
\end{eqnarray*}
for any $\Omega>0$. 
\end{theorem}

\subsection{Convergence of non-uniformly averaging under strongly convex objective functions}
In this subsection, we present the convergence result for non-uniformly averaged iterates under strongly convex functions in the following theorem.
\begin{theorem}\label{Theorem-Strong-Nonuniform-PDHG}
Denote $\beta^{k+1}$ and $\alpha^{k+1}$ as shown in Table \ref{tab:result}, then for any stationary point $\left(y^*,x^*\right)$ of problem~\eqref{prob:N}, $\left(\bar{y}^{t},\bar{x}^{t}\right)$ generated by Algorithm~\ref{ALG:SPDHG} converges to $\left(y^*,x^*\right)$ with $O(\frac{1}{t})$ rate in expectation.
\end{theorem}
We further present the high-probability result for non-uniformly averaged iterates under strongly convex functions in the following theorem.
\begin{theorem}\label{Theorem-Strong-Nonuniform-PDHG-High-Probability}
Denote $\beta^{k+1}$ and $\alpha^{k+1}$ as shown in Table \ref{tab:result}, then for any stationary point $\left(y^*,x^*\right)$ of problem~\eqref{prob:N}, { $\left(\bar{y}^{t},\bar{x}^{t}\right)$  generated by Algorithm~\ref{ALG:SPDHG} converges to $\left(y^*,x^*\right)$ with $O(\frac{1}{\sqrt{t}}+\frac{1}{t})$ rate in the sense of high probability. More specifically, the following statement holds true,}
\begin{eqnarray*}
& & \Prob\left(P(y^*,\bar{x}^t) - P(\bar{y}^t,x^*) > \frac{D_y^2}{s(t+2)} + \frac{L D_x^2}{t+2} + \frac{4\lambda_{\max}(F^\top F)D_y^2}{\mu(t+2)} + \frac{2\sqrt{2\Omega} D_x \sigma}{\sqrt{t+2}} \right. \\
& & \left. \quad \quad   + \frac{4(1+\Omega) \sigma^2}{\mu(t+2)}\right) \leq 2\exp\left(-\Omega\right),
\end{eqnarray*}
for any $\Omega>0$. 
\end{theorem}

\section{Technical Proofs}
In this section, we provide the detailed technical proof of a list of lemmas and theorems mentioned above. 
\subsection{Key lemma for high-probability analysis}
In order to make our analysis complete, we present a key technical lemma (Proposition 3.2) in \cite{Nemirovski-2009-Robust}, which is important to our iteration complexity with high probability. 
\begin{lemma}\label{Lemma-High-Probability}
For $1\leq k \leq t$, let $\zeta_k$ be a deterministic function of $\xi_{1:k}$ and satisfies that $\ex\left[\zeta_k | \xi_{\left[1:k-1\right]}\right]=0$ and $\ex\left[\exp\left(\frac{\zeta_k^2}{\sigma_k^2}\right) | \xi_{\left[1:k-1\right]}\right] \leq \exp(1)$. Then the following statements hold true, 
\begin{enumerate}
\item[(a)] Let $\gamma\geq 0$ and $1 \leq k \leq t$, then $\ex\left[\exp\left(\gamma\zeta_k\right) \ | \ \xi_{\left[1:k-1\right]}\right] \leq \exp(\gamma^2\sigma_k^2)$. 
\item[(b)] Let $S_t=\sum_{k=1}^t \zeta_k$, then $\Prob(S_t\geq\Omega\sqrt{\sum_{k=1}^t \sigma_k^2})\leq\exp\left(-\frac{\Omega^2}{4}\right)$. 
\end{enumerate}
\end{lemma}
\subsection{Convergence of uniformly averaging under convex objectives}
In this subsection, we analyze the convergence property of the SPDHG algorithm with uniformly averaged iterates for general convex objectives.
\begin{lemma}\label{Lemma-PDHG}
Let $(y^{k+1},x^{k+1})$ be generated by Algorithm \ref{ALG:SPDHG}, and $\beta^{k+1}$ and $\alpha^{k+1}$ be shown in Table \ref{tab:result}. For any stationary point $\left(y^*,x^*\right)$ of problem~\eqref{prob:N}, it holds that
\begin{eqnarray}\label{Inequality-Lemma-PDHG}
0 & \geq & \ex\left[ P(y^{k+1},x^*) - P(y^*,x^{k+1})\right] \\
& \geq & \frac{1}{2s}\left(\ex\left\| y^*-y^{k+1}\right\|^2-\ex\left\| y^*-y^k\right\|^2\right)  - \frac{\lambda_{\max}(F^\top F)D_y^2+\sigma^2}{\sqrt{k+1}} \nonumber \\
& &+ \frac{\sqrt{k+1}+L}{2}\left( \ex\left\| x^*-x^{k+1}\right\|^2-\ex\left\| x^*-x^k\right\|^2 \right). \nonumber
\end{eqnarray}
\end{lemma}
\begin{proof}
For any optimal solution $\left(y^*,x^*\right)$ of problem~\eqref{prob:N}, the first-order optimality conditions for Eq.~\eqref{LPDHG_Update_Y} and Eq.~\eqref{SPDHG_Update_X} are
\begin{align*}
0 \leq & \left( y^* - y^{k+1}\right)^\top\left(- Fx^k + \frac{1}{s} \left( y^{k+1} - y^k \right)\right), \\
0 \leq & \left( x^* - x^{k+1}\right)^\top\left[ x^{k+1} - x^k + \beta^{k+1}\left(\nabla l(x^k,\xi^{k+1}) + F^\top y^{k+1} \right)\right],
\end{align*}
{ which implies that
\begin{eqnarray}\label{inequality-key-PDHG}
& & \left( x^*-x^{k+1}\right)^\top\nabla l(x^k,\xi^{k+1})  - \left(y^*-y^{k+1}\right)^\top Fx^{k+1} \nonumber \\
& & + \left(x^*-x^{k+1}\right)^\top F^\top y^{k+1} \\
& \geq & \frac{1}{\beta^{k+1}}\left( x^* - x^{k+1}\right)^\top\left(x^k - x^{k+1}\right) + \frac{1}{s}\left( y^* - y^{k+1}\right)^\top\left( y^k - y^{k+1} \right) \nonumber \\
& &  + \left(y^*-y^{k+1}\right)^\top \left( Fx^k - Fx^{k+1} \right) \nonumber \\
& \geq & \frac{1}{2\beta^{k+1}}\left( \left\| x^*-x^{k+1}\right\|^2-\left\| x^*-x^k\right\|^2+\left\| x^{k+1}-x^k\right\|^2\right) \nonumber \\
& & + \frac{1}{2s}\left( \left\| y^*-y^{k+1}\right\|^2-\left\| y^*-y^k\right\|^2\right) + \left(y^*-y^{k+1}\right)^\top \left( Fx^k - Fx^{k+1} \right).\nonumber 
\end{eqnarray}
By using Young's inequality, we have
\begin{eqnarray}
& & \left(y^*-y^{k+1}\right)^\top \left( Fx^k - Fx^{k+1}\right) \label{inequality-repeatable1} \\ 
& = & \left(F^\top y^* - F^\top y^{k+1}\right)^\top \left( x^k - x^{k+1}\right) \nonumber \geq - \frac{\lambda_{\max}(F^\top F)D_y^2}{\gamma} - \frac{\gamma}{4}\left\| x^k - x^{k+1} \right\|^2, \nonumber 
\end{eqnarray}
for any $\gamma > 0$,} and
\begin{eqnarray*}
& & \left( x^*-x^{k+1}\right)^\top\nabla l(x^k,\xi^{k+1}) \\
& = & \left( x^*-x^{k+1}\right)^\top\nabla l(x^k) + \left( x^*-x^{k+1}\right)^\top\delta^{k+1} \\
& \leq & l(x^*)-l(x^{k+1})+\frac{L}{2}\left\| x^k - x^{k+1}\right\|^2 + \left( x^*-x^{k+1}\right)^\top\delta^{k+1} \\
& \leq & l(x^*)-l(x^{k+1}) + \left( x^*-x^k\right)^\top\delta^{k+1} + \frac{L+\sqrt{k+1}/2}{2}\left\| x^k - x^{k+1}\right\|^2 + \frac{\left\|\delta^{k+1}\right\|^2}{\sqrt{k+1}},
\end{eqnarray*}
where $\delta^{k+1}=\nabla l(x^k,\xi^{k+1})-\nabla l(x^k)$, and the first inequality holds due to Lemma 6.2 in \cite{Gao-2014-Information}. Then by letting $\gamma = \sqrt{k + 1}$ in Eq.~\eqref{inequality-repeatable1}, we obtain
\begin{eqnarray}\label{Inequality-uniform-convex}
& & l(x^*)-l(x^{k+1}) + \left(\begin{array}{c} y^* - y^{k+1} \\ x^* - x^{k+1}\end{array} \right)^{\top}\left(\begin{array}{c} - Fx^{k+1} \\ F^\top y^{k+1} \end{array} \right) \\
& \geq & \left( x^*-x^{k+1}\right)^\top\nabla l(x^k,\xi^{k+1}) - \left( x^*-x^k\right)^\top\delta^{k+1} - \frac{L+\sqrt{k+1}/2}{2}\left\| x^k - x^{k+1}\right\|^2 - \frac{\left\|\delta^{k+1}\right\|^2}{\sqrt{k+1}} \nonumber \\
& & - \left(y^*-y^{k+1}\right)^\top Fx^{k+1} + \left(x^*-x^{k+1}\right)^\top F^\top y^{k+1} \nonumber \\
& \geq & \frac{1}{2\beta^{k+1}}\left( \left\| x^*-x^{k+1}\right\|^2-\left\| x^*-x^k\right\|^2\right) + \frac{1}{2s}\left( \left\| y^*-y^{k+1}\right\|^2-\left\| y^*-y^k\right\|^2\right) - \left( x^*-x^k\right)^\top\delta^{k+1}  \nonumber \\
& & + \left(y^*-y^{k+1}\right)^\top \left( Fx^k - Fx^{k+1} \right) - \frac{\left\|\delta^{k+1}\right\|^2}{\sqrt{k+1}} + \left(\frac{1}{2\beta^{k+1}}-\frac{L+\sqrt{k+1}/2}{2}\right) \left\| x^k - x^{k+1}\right\|^2 \nonumber \\ 
& \geq & \frac{1}{2\beta^{k+1}}\left( \left\| x^*-x^{k+1}\right\|^2-\left\| x^*-x^k\right\|^2\right) + \frac{1}{2s}\left( \left\| y^*-y^{k+1}\right\|^2-\left\| y^*-y^k\right\|^2\right) - \left( x^*-x^k\right)^\top\delta^{k+1} \nonumber \\
& & - \frac{\lambda_{\max}(F^\top F)D_y^2}{\sqrt{k+1}} - \frac{\left\|\delta^{k+1}\right\|^2}{\sqrt{k+1}} + \left(\frac{1}{2\beta^{k+1}}-\frac{L+\sqrt{k+1}/2}{2}-\frac{\sqrt{k+1}}{4}\right) \left\| x^k - x^{k+1}\right\|^2 \nonumber \\ 
& = & \frac{1}{2\beta^{k+1}}\left( \left\| x^*-x^{k+1}\right\|^2-\left\| x^*-x^k\right\|^2 \right) + \frac{1}{2s}\left( \left\| y^*-y^{k+1}\right\|^2-\left\| y^*-y^k\right\|^2\right) - \left( x^*-x^k\right)^\top\delta^{k+1} \nonumber \\
& & - \frac{\lambda_{\max}(F^\top F)D_y^2}{\sqrt{k+1}} - \frac{\left\|\delta^{k+1}\right\|^2}{\sqrt{k+1}}. \nonumber
\end{eqnarray}
Since $x^k$ and $y^k$ are independent of $\xi^{k+1}$, 
we take the expectation on both sides of the above inequality conditioning on $x^k,y^k$, and conclude that
\begin{eqnarray*}
& & \ex\left[ P(y^{k+1},x^*) - P(y^*,x^{k+1})\right] \\
& \geq & \frac{1}{2\beta^{k+1}}\left( \ex\left\| x^*-x^{k+1}\right\|^2-\left\| x^*-x^k\right\|^2 \right) - \frac{\ex\left\|\delta^{k+1}\right\|^2}{\sqrt{k+1}} \\
& & + \frac{1}{2s}\left( \ex\left\| y^*-y^{k+1}\right\|^2-\left\| y^*-y^k\right\|^2\right) - \frac{\lambda_{\max}(F^\top F)D_y^2}{\sqrt{k+1}}.
\end{eqnarray*}
Finally, Eq.~\eqref{Inequality-Lemma-PDHG} follows from the above inequality and Assumption \ref{assumption_4}.
\end{proof}
\subsubsection{Proof of Theorem \ref{Theorem-PDHG-Convergence}}
{ Because $(y^k,x^k)\in\YCal\times\XCal$ and $\sum_{k=0}^t \alpha^{k+1}=1$ defined in Table~\ref{tab:result}, it holds true that $(\bar{y}^{t},\bar{x}^t)\in\YCal\times\XCal$ for all $t\geq 0$.} By invoking the convexity of function $l(\cdot)$ and summing Eq.~\eqref{Inequality-Lemma-PDHG} over $k=0,1,\ldots,t$, we have
\begin{eqnarray*}
& & \ex\left[ P(\bar{y}^t,x^*) - P(y^*,\bar{x}^t)\right] \nonumber \\
& \geq & \frac{1}{t+1}\sum\limits_{k=0}^t\left[ \frac{1}{2s}\left(\ex\left\| y^*-y^{k+1}\right\|^2-\ex\left\| y^*-y^k\right\|^2\right) \right. \nonumber \\
& & \left. +\frac{\sqrt{k+1}+L}{2}\left( \ex\left\| x^*-x^{k+1}\right\|^2 - \ex\left\| x^*-x^k\right\|^2 \right)  \right. 
 \left. -\frac{\lambda_{\max}(F^\top F)D_y^2}{\sqrt{k+1}} - \frac{\sigma^2}{\sqrt{k+1}}  \right] \nonumber \\
& \geq & -\frac{D_y^2}{2s(t+1)} - \frac{L D_x^2}{2(t+1)} - \frac{D_x^2+2\lambda_{\max}(F^\top F)D_y^2+2\sigma^2}{\sqrt{t+1}}.
\end{eqnarray*}
This together with the fact that $\ex\left[ P(\bar{y}^t,x^*) - P(y^*,\bar{x}^t)\right]\leq 0$ implies the conclusion in Theorem \ref{Theorem-PDHG-Convergence}.
\subsubsection{Proof of Theorem \ref{Theorem-PDHG-High-Probability}}
{ Summing \eqref{Inequality-uniform-convex} over $k=0,1,\ldots,t$ and invoking the convexity of function $l(\cdot)$ and the definition of $D_x$, $D_y$ and $\beta^{k+1}$ yields that} 
\begin{equation*}
P(y^*,\bar{x}^t) - P(\bar{y}^t,x^*) \leq A_t + B_t + C_t,
\end{equation*}
where 
\begin{eqnarray*}
A_t & = & \frac{D_y^2}{2s(t+1)} + \frac{L D_x^2}{2(t+1)} + \frac{D_x^2+2\lambda_{\max}(F^\top F)D_y^2}{\sqrt{t+1}}, \\ 
B_t & = & \frac{1}{t+1}\sum_{k=0}^t \left( x^*-x^k\right)^\top\delta^{k+1}, \\ 
C_t & = & \frac{1}{t+1}\sum_{k=0}^t \frac{\left\|\delta^{k+1}\right\|^2}{\sqrt{k+1}}.
\end{eqnarray*}
Note that random variables $B_t$ and $C_t$ are dependent on $\xi_{[1:k-1]}$. We have the following two claims.
\begin{enumerate}
\item For $\Omega>0$, we have
\begin{equation*}
\Prob\left(B_t>\frac{2\sqrt{\Omega} D_x \sigma}{\sqrt{t+1}}\right) \leq \exp\left(-\Omega\right). 
\end{equation*}
\item For $\Omega>0$, we have
\begin{equation*}
\Prob\left(C_t>\frac{(1+\Omega) \sigma^2}{\sqrt{t+1}}\right) \leq \exp\left(-\Omega\right). 
\end{equation*}
\end{enumerate}
For claim 1, we use Lemma \ref{Lemma-High-Probability} by setting $\zeta_k=\left(x^*-x^k\right)^\top\delta^{k+1}$, and $S_t=\sum_{k=0}^t\zeta_k$, and $\sigma_k=D_x\sigma$,  we can verify that $\ex\left[\zeta_k | \xi_{[1:k-1]}\right]=0$, and 
\begin{eqnarray*}
\ex\left[\exp\left(\frac{\zeta_k^2}{\sigma_k^2}\right) | \xi_{\left[1:k-1\right]}\right] & \leq & \ex\left[\exp\left(\frac{D_x^2\left\|\delta^{k+1}\right\|^2}{\sigma_k^2}\right) | \xi_{\left[1:k-1\right]}\right] \\
& \leq & \exp(1),
\end{eqnarray*}
where the first inequality holds true due to the fact that $\exp(\cdot)$ is a convex function together with $\|x^*-x^k\|\le D_x^2$.
Based on the above results, it follows that
\begin{equation*}
\Prob\left(S_t>\Omega_1 D_x \sigma \sqrt{t+1}\right) \leq \exp\left(-\frac{\Omega_1^2}{4}\right). 
\end{equation*}
Since $S_t = (t+1) B_t$ and $\Omega_1=2\sqrt{\Omega}$, we have
\begin{equation*}
\Prob\left(B_t>\frac{2\sqrt{\Omega} D_x \sigma}{\sqrt{t+1}}\right) \leq \exp\left(-\Omega\right). 
\end{equation*}
For claim 2, we let $\eta_{k+1}=\frac{\frac{1}{\sqrt{k+1}}}{\sum_{i=0}^t \frac{1}{\sqrt{k+1}}}$, and obtain $\eta_k\in\left(0,1\right)$ and $\sum_{k=0}^t\eta_k=1$. Since $\left\{\delta^{k+1}\right\}_{k=0}^t$ are independent and
applying Assumption \ref{assumption_4}, we have
\begin{eqnarray*}
& & \ex\left[\exp\left(\frac{\sum_{k=0}^t \eta_{k+1}\left\|\delta^{k+1}\right\|^2}{\sigma^2}\right) \right] \\
& = & \prod_{k=0}^t \ex\left[\exp\left(\frac{\eta_{k+1}\left\|\delta^{k+1}\right\|^2}{\sigma^2}\right) \right] \\
& \leq & \prod_{k=0}^t \left(\ex\left[\exp\left(\frac{\left\|\delta^{k+1}\right\|^2}{\sigma^2}\right) \right]\right)^{\eta_{k+1}} \\
& \leq & \prod_{k=0}^t\left(\exp(1)\right)^{\eta_{k+1}} = \exp\left(\sum_{k=0}^t \eta_{k+1}\right) = \exp(1). 
\end{eqnarray*}
By Markov's Inequality, we have 
\begin{eqnarray*}
& & \Prob\left(C_t>\frac{(1+\Omega) \sigma^2}{\sqrt{t+1}}\right) \\
& \leq & \Prob\left(C_t>\frac{(1+\Omega) \sigma^2}{2(t+1)}\sum_{k=0}^t\frac{1}{\sqrt{k+1}}\right) \\
& \leq & \exp\left(-(1+\Omega)\right) \ex\left[\exp\left(\frac{\sum_{k=0}^t \eta_{k+1}\left\|\delta^{k+1}\right\|^2}{\sigma^2}\right) \right] \\
& \leq & \exp\left(-\Omega\right). 
\end{eqnarray*}
In conclusion, we have the desired inequality.

\subsection{Convergence of uniformly averaging under strongly convex objectives}
In this subsection, we analyze the convergence property of the SPDHG algorithm with uniformly averaged iterates for strongly convex objectives.
\begin{lemma}\label{Lemma-Strong-PDHG}
Let $(y^{k+1},x^{k+1})$ be generated by Algorithm \ref{ALG:SPDHG}, and $\beta^{k+1}$ and $\alpha^{k+1}$ be shown in Table \ref{tab:result}. For any any stationary point $\left(y^*,x^*\right)$ of problem~\eqref{prob:N}, it holds that
\begin{eqnarray}\label{Inequality-Lemma-Strong-PDHG}
0 & \geq & \ex\left[ P(y^{k+1},x^*) - P(y^*,x^{k+1}) \right] \\
& \geq & \frac{\mu(k+1)+L}{2}\ex\left\| x^*-x^{k+1}\right\|^2 + \frac{1}{2s}\ex\left\| y^*-y^{k+1}\right\|^2 \nonumber \\
& & -\frac{\mu k+L}{2}\ex\left\| x^*-x^k\right\|^2 - \frac{1}{2s}\ex\left\| y^*-y^k\right\|^2 - \frac{\lambda_{\max}(F^\top F)D_y^2+\sigma^2}{\mu(k+1)}. \nonumber
\end{eqnarray}
\end{lemma}
\begin{proof}
By using the same argument as Lemma \ref{Lemma-PDHG} and the strongly convexity of $l$, we have
\begin{eqnarray}\label{inequality-repeatable2}
& & \left( x^*-x^{k+1}\right)^\top\nabla l(x^k,\xi^{k+1}) \\
& \leq & l(x^*)- l (x^k) - \frac{\mu}{2}\left\| x^* - x^k\right\|^2 + l(x^k) - l(x^{k+1}) + \frac{L}{2}\left\| x^k - x^{k+1}\right\|^2  + \left( x^*-x^{k+1}\right)^\top\delta^{k+1} \nonumber \\
& \leq & l(x^*)-l(x^{k+1}) + \left( x^*-x^k\right)^\top\delta^{k+1} - \frac{\mu}{2}\left\| x^* - x^k\right\|^2 + \frac{L}{2}\left\| x^k - x^{k+1}\right\|^2 \nonumber \\
& & + \frac{\kappa}{4}\left\| x^k - x^{k+1}\right\|^2 + \frac{1}{\kappa}\left\|\delta^{k+1}\right\|^2. \nonumber 
\end{eqnarray}
Substituting Eq.~\eqref{inequality-repeatable1} with $\gamma = \mu(k+1)$ and Eq.~\eqref{inequality-repeatable2} with $\kappa = \mu(k+1)$ into Eq.~\eqref{inequality-key-PDHG} yields that
\begin{eqnarray}\label{Inequality-uniform-strongly-convex}
& & l(x^*)-l(x^{k+1}) + \left(\begin{array}{c} y^* - y^{k+1} \\ x^* - x^{k+1}\end{array} \right)^{\top}\left(\begin{array}{c} - Fx^{k+1} \\ F^\top y^{k+1} \end{array} \right) \\
& \geq & \frac{1}{2s}\left(\left\| y^*-y^{k+1}\right\|^2 - \left\| y^*-y^k\right\|^2\right) - \frac{\left\|\delta^{k+1}\right\|^2}{\mu(k+1)} + \frac{\mu(k+1)+L}{2}\left\| x^*-x^{k+1}\right\|^2 - \frac{\mu k + L}{2}\left\| x^*-x^k\right\|^2 \nonumber \\
& &  + \left(\frac{1}{2\beta^{k+1}} - \frac{L+\mu(k+1)}{2}\right)\left\| x^k - x^{k+1}\right\|^2 - \frac{\lambda_{\max}(F^\top F)D_y^2}{\mu(k+1)} - \left( x^*-x^k\right)^\top\delta^{k+1}. \nonumber 
\end{eqnarray}
Then we obtain Eq.~\eqref{Inequality-Lemma-Strong-PDHG} as the same as that in Lemma \ref{Lemma-PDHG}.
\end{proof}
\subsubsection{Proof of Theorem \ref{Theorem-Strong-PDHG-Convergence}}
{ Because $(y^k,x^k)\in\YCal\times\XCal$, it holds true that $(\bar{y}^{t},\bar{x}^t)\in\YCal\times\XCal$ for all $t\geq 0$ since $\sum_{k=0}^t \alpha^{k+1}=1$ where $\alpha^{k+1}$ is defined in Table~\ref{tab:result}.} By invoking the convexity of function $l(\cdot)$ and summing Eq.~\eqref{Inequality-Lemma-Strong-PDHG} over $k=0,1,\ldots,t$, we have
\begin{eqnarray*}
& & \ex\left[ P(\bar{y}^t,x^*) - P(y^*,\bar{x}^t)\right] \nonumber \\
& \geq & \frac{1}{t+1}\sum\limits_{k=0}^t\left[ \frac{1}{2s}\left( \ex\left\| y^*-y^{k+1}\right\|^2 - \ex\left\| y^*-y^k\right\|^2\right) \right. \nonumber \\
& & \left. + \frac{\mu(k+1)+L}{2}\left\| x^*-x^{k+1}\right\|^2 -\frac{\mu k+L}{2}\left\| x^*-x^k\right\|^2 
 - \frac{\lambda_{\max}(F^\top F)D_y^2+\sigma^2}{\mu(k+1)} \right] \nonumber \\
& \geq & -\frac{D_y^2}{2s(t+1)} - \frac{LD_x^2}{2(t+1)} - \frac{\left(\lambda_{\max}(F^\top F)D_y^2+\sigma^2\right)\log(t+1)}{\mu(t+1)}.
\end{eqnarray*}
This together with the fact that $\ex\left[ P(\bar{y}^t,x^*) - P(y^*,\bar{x}^t)\right]\leq 0$ implies the conclusion in Theorem  \ref{Theorem-Strong-PDHG-Convergence}.
\subsubsection{Proof of Theorem \ref{Theorem-Strong-PDHG-High-Probability}}
{ Summing \eqref{Inequality-uniform-strongly-convex} over $k=0,1,\ldots,t$ and invoking the convexity of function $l(\cdot)$ and the definition of $D_x$, $D_y$ and $\beta^{k+1}$ yields that} 
\begin{equation*}
P(y^*,\bar{x}^t) - P(\bar{y}^t,x^*) \leq A_t + B_t + C_t,
\end{equation*}
where 
\begin{eqnarray*}
A_t & = & \frac{D_y^2}{2s(t+1)} + \frac{L D_x^2}{2(t+1)} + \frac{\lambda_{\max}(F^\top F)D_y^2\log(t+1)}{\mu(t+1)}, \\ 
B_t & = & \frac{1}{t+1}\sum_{k=0}^t \left( x^*-x^k\right)^\top\delta^{k+1}, \\ 
C_t & = & \frac{1}{t+1}\sum_{k=0}^t \frac{\left\|\delta^{k+1}\right\|^2}{\mu(k+1)}.
\end{eqnarray*}
Note that random variables $B_t$ and $C_t$ are dependent on $\xi_{[1:k-1]}$. We have the following two claims.
\begin{enumerate}
\item For $\Omega>0$, we have
\begin{equation*}
\Prob\left(B_t>\frac{2\sqrt{\Omega} D_x \sigma}{\sqrt{t+1}}\right) \leq \exp\left(-\Omega\right). 
\end{equation*}
\item For $\Omega>0$, we have
\begin{equation*}
\Prob\left(C_t>\frac{(1+\Omega) \sigma^2 \log(t+1)}{2\mu(t+1)}\right) \leq \exp\left(-\Omega\right). 
\end{equation*}
\end{enumerate}
For claim 1, we apply the similar argument used in Theorem \ref{Theorem-PDHG-High-Probability}, and obtain the desired inequality. For claim 2, we let $\eta_{k+1}=\frac{\frac{1}{k+1}}{\sum_{i=0}^t \frac{1}{k+1}}$, and apply the similar argument used in Theorem \ref{Theorem-PDHG-High-Probability}, we have
\begin{equation*}
\ex\left[\exp\left(\frac{\sum_{k=0}^t \eta_{k+1}\left\|\delta^{k+1}\right\|^2}{\sigma^2}\right) \right] \leq \exp(1). 
\end{equation*}
By Markov's Inequality, we have 
\begin{eqnarray*}
& & \Prob\left(C_t>\frac{(1+\Omega) \sigma^2\log(t+1)}{2\mu(t+1)}\right) \\
& \leq & \Prob\left(C_t>\frac{(1+\Omega) \sigma^2}{2(t+1)}\sum_{k=0}^t\frac{1}{\mu(k+1)}\right) \\
& \leq & \exp\left(-(1+\Omega)\right) \ex\left[\exp\left(\frac{\sum_{k=0}^t \eta_{k+1}\left\|\delta^{k+1}\right\|^2}{\sigma^2}\right) \right] \\
& \leq & \exp\left(-\Omega\right). 
\end{eqnarray*}
In conclusion, we have the desired inequality.

\subsection{Convergence of non-uniformly averaging under strongly convex objectives}
In this subsection, we analyze the convergence property of the SPDHG algorithm with non-uniformly averaged iterates for strongly convex objectives.
\begin{lemma}\label{Lemma-Strong-Nonuniform-PDHG}
Let $(y^{k+1},x^{k+1})$ be generated by Algorithm \ref{ALG:SPDHG}, and $\beta^{k+1}$ and $\alpha^{k+1}$ be shown in Table \ref{tab:result}. For any stationary point $\left(y^*,x^*\right)$ of problem~\eqref{prob:N}, it holds that
\begin{eqnarray}\label{Inequality-Lemma-Strong-Nonuniform-PDHG}
0 & \geq & \ex\left[ P(y^{k+1},x^*) - P(y^*,x^{k+1}) \right] \\
& \geq & \frac{\mu(k+2)+2L}{4}\ex\left\| x^*-x^{k+1}\right\|^2 + \frac{1}{2s}\ex\left\| y^*-y^{k+1}\right\|^2 \nonumber \\
& & -\frac{\mu k+2L}{4}\ex\left\| x^*-x^k\right\|^2 - \frac{1}{2s}\ex\left\| y^*-y^k\right\|^2 - \frac{2\lambda_{\max}(F^\top F)D_y^2+2\sigma^2}{\mu(k+1)}. \nonumber
\end{eqnarray}
\end{lemma}
\begin{proof}
By substituting Eq.~\eqref{inequality-repeatable1} with $\gamma = \frac{\mu(k+1)}{2}$ and Eq.~\eqref{inequality-repeatable2} with $\kappa = \frac{\mu(k+1)}{2}$ into Eq.~\eqref{inequality-key-PDHG}, we have
\begin{eqnarray*}
 \left(y^*-y^{k+1}\right)^\top \left( Fx^k - Fx^{k+1}\right)
  \geq  - \frac{2\lambda_{\max}(F^\top F)D_y^2}{\mu(k+1)} - \frac{\mu(k+1)}{8}\left\| x^k - x^{k+1} \right\|^2,
\end{eqnarray*}
and
\begin{eqnarray*}
& & \left( x^*-x^{k+1}\right)^\top\nabla l(x^k,\xi^{k+1}) \\
& \leq & l(x^*)-l(x^{k+1}) + \left( x^*-x^k\right)^\top\delta^{k+1} - \frac{\mu}{2}\left\| x^* - x^k\right\|^2  \nonumber \\
& & + \frac{L}{2}\left\| x^k - x^{k+1}\right\|^2 + \frac{\mu(k+1)}{8}\left\| x^k - x^{k+1}\right\|^2 + \frac{2}{\mu(k+1)}\left\|\delta^{k+1}\right\|^2.
\end{eqnarray*}
Then we plug the above two inequalities into Eq.~\eqref{inequality-key-PDHG}, and then follow the same argument as Lemma \ref{Lemma-Strong-PDHG} to obtain the desired inequality in Eq.~\eqref{Inequality-Lemma-Strong-Nonuniform-PDHG}.
\end{proof}
\subsubsection{Proof of Theorem \ref{Theorem-Strong-Nonuniform-PDHG}}
{ Because $(y^k,x^k)\in\YCal\times\XCal$, it holds true that $(\bar{y}^{t},\bar{x}^t)\in\YCal\times\XCal$ for all $t\geq 0$ since $\sum_{k=0}^t \alpha^{k+1}=1$ where $\alpha^{k+1}$ is defined in Table~\ref{tab:result}.} By invoking the convexity of function $l(\cdot)$ and summing Eq.~\eqref{Inequality-Lemma-Strong-Nonuniform-PDHG} over $k=0,1,\ldots,t$, we have
\begin{eqnarray*}
& & \ex\left[ P(\bar{y}^t,x^*) - P(y^*,\bar{x}^t)\right] \nonumber \\
& \geq & \frac{2}{(t+1)(t+2)}\sum\limits_{k=0}^t(k+1)\left[  - \frac{2\lambda_{\max}(F^\top F)D_y^2+2\sigma^2}{\mu(k+1)}  + \frac{\mu(k+2)+2L}{4}\left\| x^*-x^{k+1}\right\|^2 \right. \nonumber \\
& & \left. -\frac{\mu k+2L}{4}\left\| x^*-x^k\right\|^2 + \frac{1}{2s}\left( \ex\left\| y^*-y^{k+1}\right\|^2 - \ex\left\| y^*-y^k\right\|^2\right)  \right] \nonumber \\
& \geq & -\frac{D_y^2}{s(t+2)} - \frac{LD_x^2}{t+2} - \frac{4\lambda_{\max}(F^\top F)D_y^2+4\sigma^2}{\mu(t+2)} \nonumber \\
& & + \frac{\mu}{2(t+1)(t+2)}\sum\limits_{k=0}^t \left[(k+2)(k+1)\left\| x^*-x^{k+1}\right\|^2  - (k+1)k\left\| x^*-x^k\right\|^2 \right].
\end{eqnarray*}
Therefore, we conclude that
\begin{eqnarray*}
0 & \geq & \ex\left[ P(\bar{y}^t,x^*) - P(y^*,\bar{x}^t)\right] \\
& \geq & -\frac{D_y^2}{s(t+2)} - \frac{LD_x^2}{t+2} - \frac{4\lambda_{\max}(F^\top F)D_y^2+4\sigma^2}{\mu(t+2)},
\end{eqnarray*}
which implies the conclusion in Theorem \ref{Theorem-Strong-Nonuniform-PDHG}.
\subsubsection{Proof of Theorem \ref{Theorem-Strong-Nonuniform-PDHG-High-Probability}}
{ By using the same argument as Theorem~\ref{Theorem-Strong-PDHG-High-Probability}, we have}
\begin{equation*}
P(y^*,\bar{x}^t) - P(\bar{y}^t,x^*) \leq A_t + B_t + C_t,
\end{equation*}
where 
\begin{eqnarray*}
A_t & = & \frac{D_y^2}{s(t+2)} + \frac{L D_x^2}{t+2} + \frac{4\lambda_{\max}(F^\top F)D_y^2}{\mu(t+2)}, \\ 
B_t & = & \frac{2}{(t+1)(t+2)}\sum_{k=0}^t (k+1)\left( x^*-x^k\right)^\top\delta^{k+1}, \\ 
C_t & = & \frac{4}{\mu(t+1)(t+2)}\sum_{k=0}^t \left\|\delta^{k+1}\right\|^2.
\end{eqnarray*}
Note that random variables $B_t$ and $C_t$ are dependent on $\xi_{[1:k-1]}$. We have the following two claims.
\begin{enumerate}
\item For $\Omega>0$, we have
\begin{equation*}
\Prob\left(B_t>\frac{2\sqrt{2\Omega} D_x \sigma}{\sqrt{t+2}}\right) \leq \exp\left(-\Omega\right). 
\end{equation*}
\item For $\Omega>0$, we have
\begin{equation*}
\Prob\left(C_t>\frac{(1+\Omega) \sigma^2}{\sqrt{t+1}}\right) \leq \exp\left(-\Omega\right). 
\end{equation*}
\end{enumerate}
For claim 1, we use Lemma \ref{Lemma-High-Probability} by setting $\zeta_k=(k+1)\left(x^*-x^k\right)^\top\delta^{k+1}$, and $S_t=\sum_{k=0}^t\zeta_k$, and $\sigma_k=(k+1)D_x\sigma$,  we can verify that $\ex\left[\zeta_k | \xi_{[1:k-1]}\right]=0$, and apply the similar argument used in Theorem \ref{Theorem-PDHG-High-Probability}, we have 
\begin{eqnarray*}
\ex\left[\exp\left(\frac{\zeta_k^2}{\sigma_k^2}\right) | \xi_{\left[1:k-1\right]}\right] & \leq & \ex\left[\exp\left(\frac{D_x^2\left\|\delta^{k+1}\right\|^2}{\sigma_k^2}\right) | \xi_{\left[1:k-1\right]}\right] \\
& \leq & \exp(1). 
\end{eqnarray*}
Based on the above results, it follows that
\begin{equation*}
\Prob\left(S_t>\Omega_1 D_x \sigma \sqrt{\frac{(t+1)(t+2)(2t+3)}{6}}\right) \leq \exp\left(-\frac{\Omega_1^2}{4}\right). 
\end{equation*}
Since $S_t = \frac{(t+1)(t+2)}{2} B_t$ and $\Omega_1=2\sqrt{\Omega}$, we have
\begin{equation*}
\Prob\left(B_t>\frac{2\sqrt{2\Omega} D_x \sigma}{\sqrt{t+2}}\right) \leq \exp\left(-\Omega\right). 
\end{equation*}
For claim 2, we let $\eta_{k+1}=\frac{1}{t+1}$, and obtain that 
\begin{eqnarray*}
& & \ex\left[\exp\left(\frac{\sum_{k=0}^t \eta_{k+1}\left\|\delta^{k+1}\right\|^2}{\sigma^2}\right) \right] \\
& = & \prod_{k=0}^t \ex\left[\exp\left(\frac{\eta_{k+1}\left\|\delta^{k+1}\right\|^2}{\sigma^2}\right) \right] \\
& \leq & \prod_{k=0}^t \left(\ex\left[\exp\left(\frac{\left\|\delta^{k+1}\right\|^2}{\sigma^2}\right) \right]\right)^{\eta_{k+1}} \\
& \leq & \prod_{k=0}^t\left(\exp(1)\right)^{\eta_{k+1}} = \exp\left(\sum_{k=0}^t \eta_{k+1}\right) = \exp(1). 
\end{eqnarray*}
By Markov's Inequality, we have 
\begin{eqnarray*}
& & \Prob\left(C_t>\frac{4(1+\Omega) \sigma^2}{\mu(t+2)}\right) \\
& \leq & \Prob\left(C_t>\frac{4(1+\Omega) \sigma^2}{\mu(t+1)(t+2)}\sum_{k=0}^t 1\right) \\
& \leq & \exp\left(-(1+\Omega)\right) \ex\left[\exp\left(\frac{\sum_{k=0}^t \eta_{k+1}\left\|\delta^{k+1}\right\|^2}{\sigma^2}\right) \right] \\
& \leq & \exp\left(-\Omega\right). 
\end{eqnarray*}
In conclusion, we have the desired inequality.

\begin{table}[h]
\caption{Statistics of datasets.}\label{tab:data}
\begin{center}
\begin{tabular}{c|c|c} \hline
Dataset & Number of Samples & Dimension \\ \hline  
\textit{splice} & 1,000 & 60 \\ 
\textit{svmguide3} & 1,243 & 21\\
\textit{hitech} & 2,301 & 10,080 \\
\textit{la12} & 2,301 & 31,472 \\
\textit{k1b} & 2,340 & 21,839 \\ 
\textit{ng3sim} & 2,998 & 15,810 \\ 
\textit{la2} & 3,075 & 31,472 \\
\textit{la1} & 3,204 & 31,472 \\
\textit{reviews} & 4,069 & 18,482 \\
\textit{classic} & 7,094 & 41,681 \\
\textit{sports} & 8,580 & 14,866 \\
\textit{ohscal} & 11,162  & 11,465 \\      
\textit{20news} & 16,242 & 100 \\
\textit{a9a} & 32,561 & 123 \\
\textit{w8a} & 64,700 & 300 \\
\textit{covtype} & 581,012 & 55 \\
\textit{SUSY} & 5,000,000 & 19 \\
\hline
\end{tabular}
\end{center}
\end{table}

\section{Experiments}\label{sec:experiment}
We conduct experiments by evaluating two models: graph-guided logistic regression (GGLR)~\eqref{Prob:GGLR} and graph-guided regularized logistic regression (GGRLR)~\eqref{Prob:GGRLR} in \cite{Zhong-2014-Fast},
\begin{equation}\label{Prob:GGLR}
\min_{x\in \XCal} \ l(x) + \lambda\|Fx\|_1
\end{equation}
and
\begin{equation}\label{Prob:GGRLR}
\min_{x\in \XCal} \ l(x) + \frac{\gamma}{2}\left\|x\right\|_2^2 + \lambda\|Fx\|_1.
\end{equation}
Here $l(x) = \frac{1}{N}\left[\sum\limits_{i=1}^N l(x,\xi_i)\right]$ is empirical average of $l(x,\xi_i)$ on a set of samples, and $l(x,\xi_i)$ is logistic function $\log\left( 1+\exp\left(-b_i \cdot a_i^\top x\right) \right)$, where $\xi_i = (a_i,b_i)$. $\lambda$ is the regularization parameter. 
$F$ is a penalty matrix promoting the desired sparse structure of $x$, which is generated by sparse inverse covariance selection \cite{Scheinberg-2010-Sparse}. 
To proceed, we reformulate problems~\eqref{Prob:GGLR} and~\eqref{Prob:GGRLR} into the convex-concave saddle point problem~\eqref{prob:N} and apply our proposed SPDHG algorithm. 
On the other hand, we can reformulate problems~\eqref{Prob:GGLR} and~\eqref{Prob:GGRLR} into problem~\eqref{prob-ADMM} by introducing an additional variable $z=Fx$ and then apply stochastic ADMM algorithms.

In the experiments, we compare our SPDHG algorithm with the LPDHG algorithm, and six existing stochastic ADMM algorithms\footnote{We use the code of SADMM, OPG-ADMM, RDA-ADMM and FSADMM provided by the authors while implementing two variants of adaptive SADMM according to \cite{Zhao-2015-Adaptive}.}: SADMM \cite{Ouyang-2013-Stochastic}, OPG-ADMM \cite{Suzuki-2013-Dual}, RDA-ADMM \cite{Suzuki-2013-Dual}, FSADMM\cite{Zhong-2014-Fast}, and two variants of adaptive SADMM (\textit{i.e.}, SADMMdiag and SADMMfull) \cite{Zhao-2015-Adaptive}. 
We do not include online ADMM \cite{Wang-2012-Online} and SDCA-ADMM \cite{Suzuki-2014-Stochastic} since \cite{Suzuki-2013-Dual} has shown that RDA-ADMM performs better than online ADMM while \cite{Zhang-2015-Stochastic} has shown that the performance of FSADMM is comparable to that of SDCA-ADMM. 
Finally, SPDC and Adaptive SPDC are excluded from the experiments since they cannot solve problem~\eqref{Prob:GGLR} and problem~\eqref{Prob:GGRLR}, as clarified in Section \ref{sec:introduction}.

\begin{figure}
\center
{\includegraphics[width=0.8\textwidth]{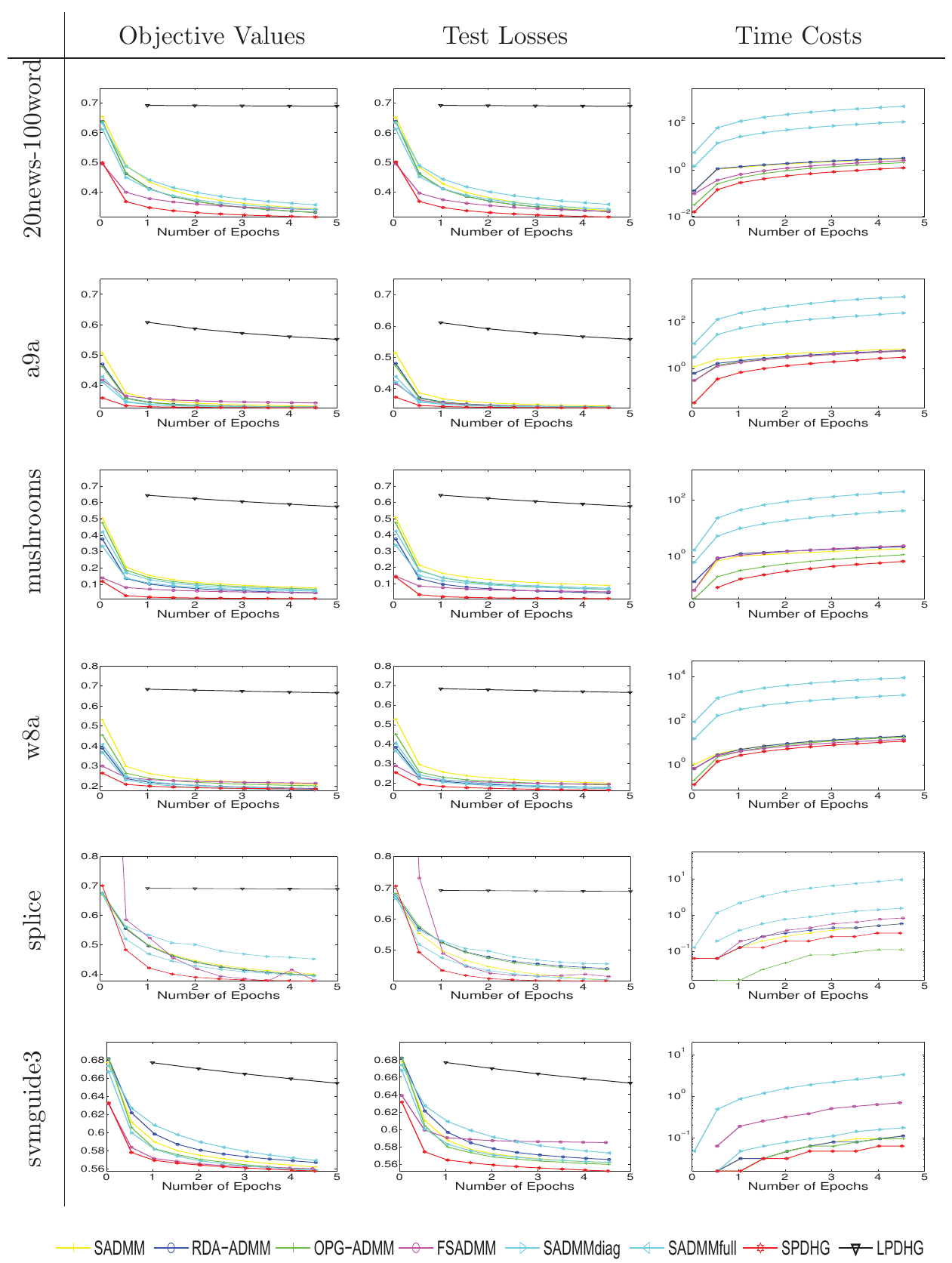}}
\caption{Comparison of SPDHG with STOC-ADMM (SADMM), RDA-ADMM, OPG-ADMM, Fast-SADMM (FSADMM), Ada-SADMMdiag, Ada-SADMMfull and LPDHG on \textbf{Graph-Guided Logistic Regression} Task. \textbf{Left Panels}: Averaged objective values. \textbf{Middle Panels}: Averaged test losses.
\textbf{Right Panels}: Averaged time costs (in seconds).}
\label{fig-GGLR-complexity}
\end{figure}

\begin{figure}
\center
{\includegraphics[width=0.8\textwidth]{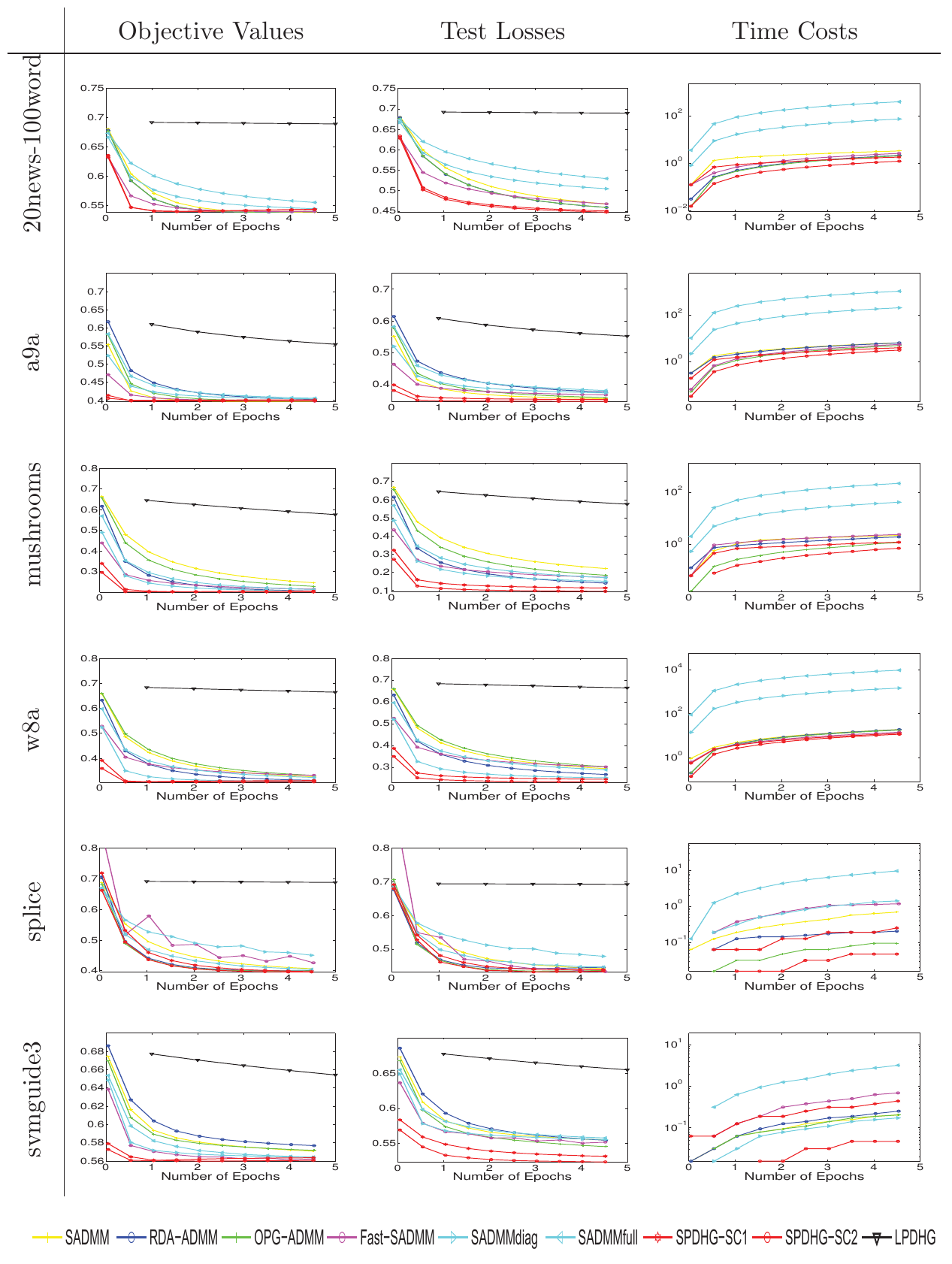}}
    \caption{Comparison of SPDHG-SC1 (Uniformly Averaged) and SPDHG-SC2 (Non-Uniformly Averaged) with STOC-ADMM (SADMM), RDA-ADMM, OPG-ADMM, Fast-SADMM (FSADMM), Ada-SADMMdiag, Ada-SADMMfull and LPDHG on \textbf{Graph-Guided Regularized Logistic Regression} Task. \textbf{Left Panels}: Averaged objective values. \textbf{Middle Panels}: Averaged test losses. \textbf{Right Panels}: Averaged time costs (in seconds).}
\label{fig-GGRLR-complexity}
\end{figure}

\begin{figure}
\center
{\includegraphics[width=0.8\textwidth]{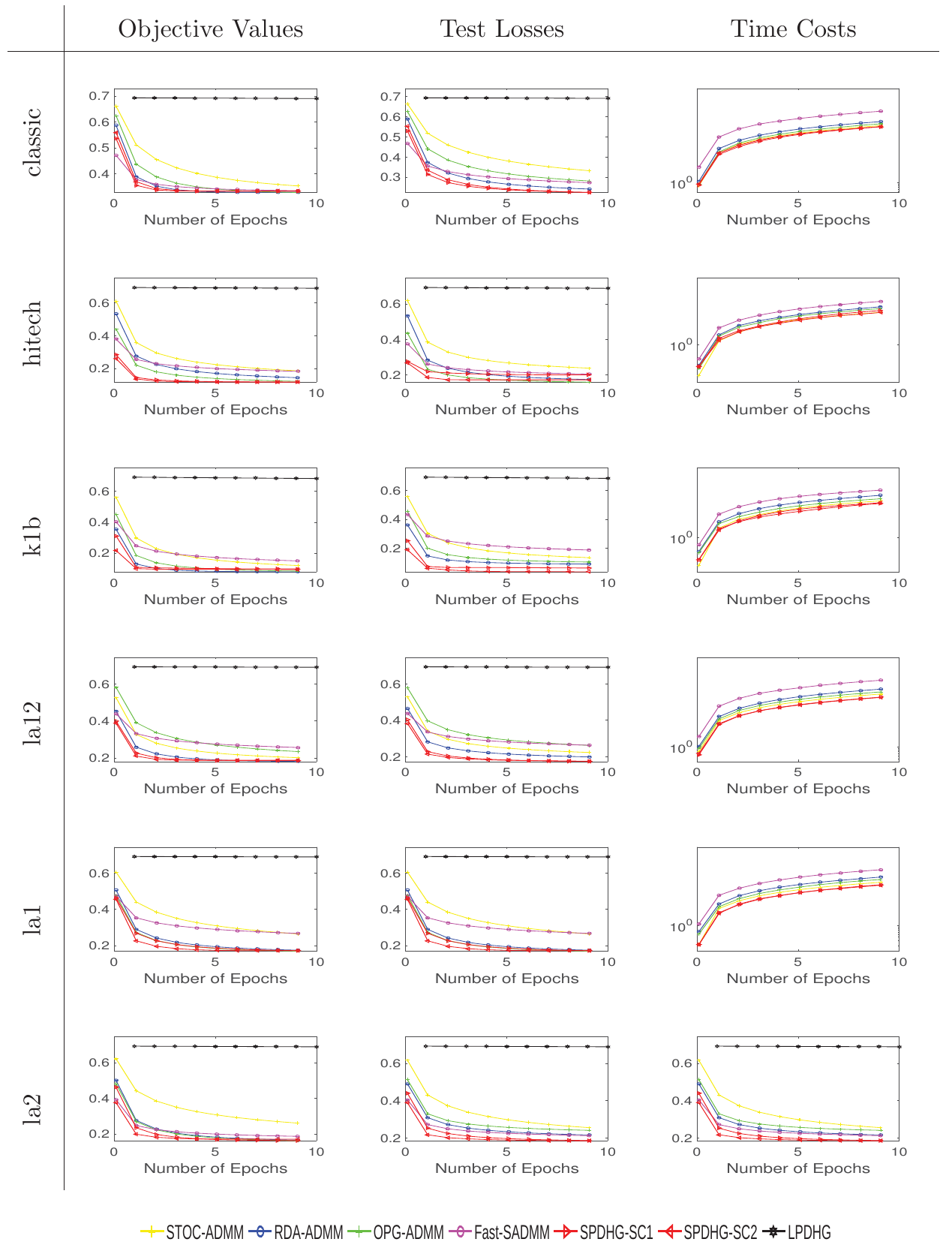}}
    \caption{Comparison of SPDHG-SC1 (Uniformly Averaged) and SPDHG-SC2 (Non-Uniformly Averaged) with STOC-ADMM (SADMM), RDA-ADMM, OPG-ADMM, Fast-SADMM (FSADMM), Ada-SADMMdiag, Ada-SADMMfull and LPDHG on \textbf{Graph-Guided Regularized Logistic Regression} Task. \textbf{Left Panels}: Averaged objective values. \textbf{Middle Panels}: Averaged test losses. \textbf{Right Panels}: Averaged time costs (in seconds).}
\label{fig-GGRLR-complexity-2}
\end{figure}

\begin{figure}
\center
{\includegraphics[width=0.8\textwidth]{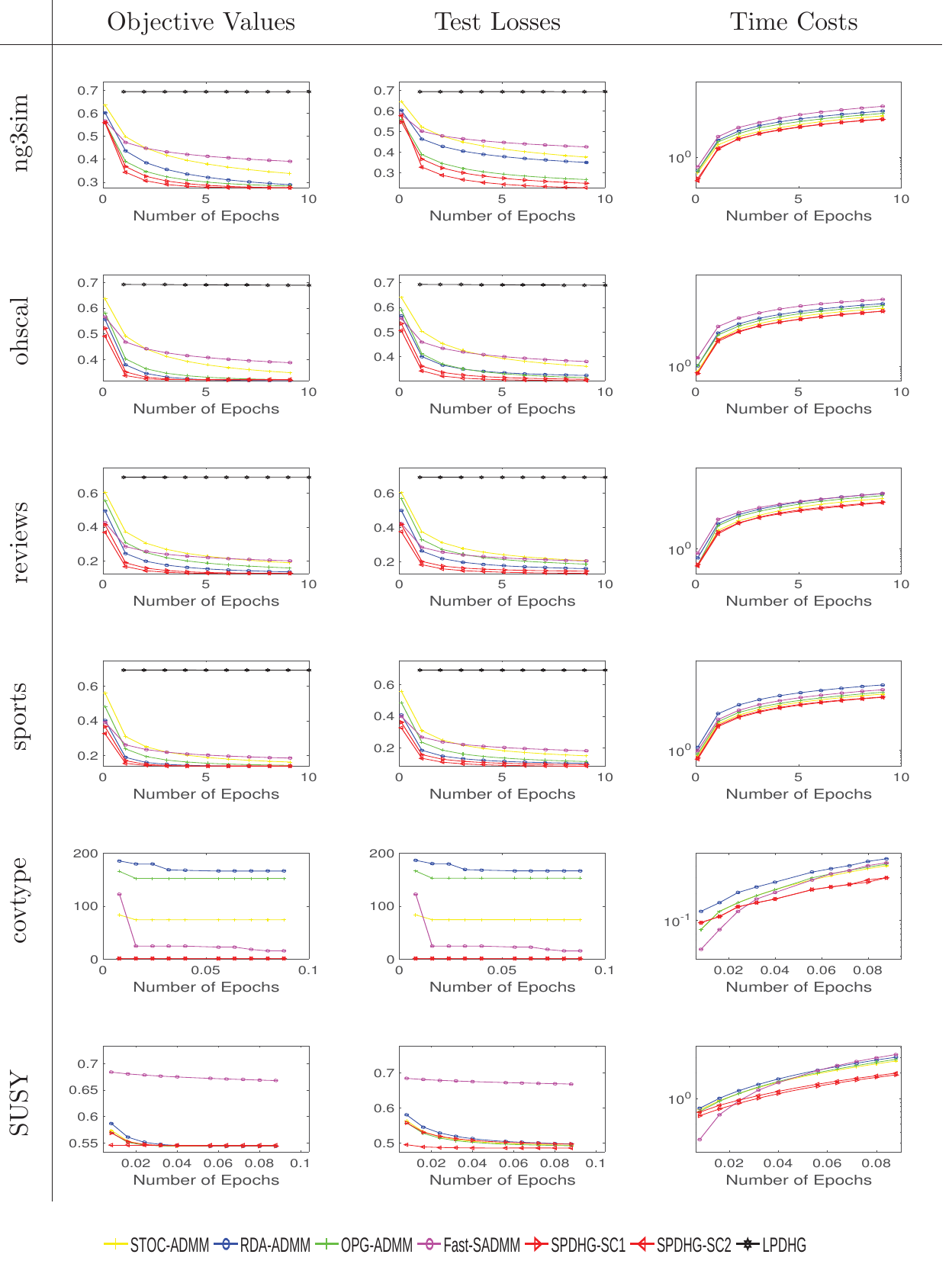}}
    \caption{Comparison of SPDHG-SC1 (Uniformly Averaged) and SPDHG-SC2 (Non-Uniformly Averaged) with STOC-ADMM (SADMM), RDA-ADMM, OPG-ADMM, Fast-SADMM (FSADMM), Ada-SADMMdiag, Ada-SADMMfull and LPDHG on \textbf{Graph-Guided Regularized Logistic Regression} Task. \textbf{Left Panels}: Averaged objective values. \textbf{Middle Panels}: Averaged test losses. \textbf{Right Panels}: Averaged time costs (in seconds).}
\label{fig-GGRLR-complexity-3}
\end{figure}

The experiments are conducted on 17 binary classification datasets: \textit{classic}, \textit{hitech}, \textit{k1b}, \textit{la12}, \textit{la1}, \textit{la2}, \textit{reviews}, \textit{sports}, \textit{ohscal} \footnote{https://www.shi-zhong.com/software/docdata.zip}, \textit{a9a}, \textit{20news} \footnote{www.cs.nyu.edu/roweis/data.html}, \textit{splice}, \textit{svmguide3}, \textit{mushrooms}, \textit{w8a}, \textit{covtype}, \textit{SUSY} \footnote{https://www.csie.ntu.edu.tw/\~{}cjlin/libsvm/}. On each dataset, we use 80\% samples for training and 20\% for testing, and calculate the lipschitz constant $L$ as its classical upper bound $\hat{L} = 0.25\max_{1\leq i\leq n}\|a_i\|^2$. 
The regularization parameters are set to $\lambda = 10^{-5}$ and $\gamma = 10^{-2}$. 
To reduce statistical variability, experimental results are averaged over 10 repetitions. 
We set the parameters of SPDHG exactly following our theory while using cross validation to select the parameters of the other algorithms. 
Additionally, we use the metrics in \cite{Zhong-2014-Fast} to compare our algorithm with the other algorithms, including test losses, objective values and time costs to compare our algorithm with the other. The ``test loss" means the value of the empirically averaged loss evaluated on a test dataset, while the ``objective value" means the sum of the empirically averaged loss and regularized terms evaluated on a training dataset, and the ``time cost'' means the computational time consumption of each algorithm. { In addition, ``Number of Epochs" for the horizontal axis is the ratio between iteration number and data size.}

Specifically, we use test losses (\textit{i.e.}, $l(x)$) on test datasets, objective values (\textit{i.e.}, $l(x)+\lambda \|Fx\|_1$ on the GGLR task and $l(x)+\frac{\gamma}{2}\|x\|_2^2+\lambda \|Fx\|_1$ on the GGRLR task) on training datasets, and computational time costs on training datasets. Figure \ref{fig-GGLR-complexity} shows the objective values, test losses and time costs as the function of the number of epochs on the GGLR task, where the objective function is convex but not necessarily strongly convex. We observe that our algorithm SPDHG mostly achieves the best performance, surpassing six stochastic ADMM algorithms, all of which outperform LPDHG by a significant margin. FSADMM sometimes achieves better solutions but consumes much more computational time than SPDHG. In fact, our algorithm requires the least iterations and computational time among all the evaluated algorithms. Furthermore, the performance of our algorithm SPDHG on 17 datasets is most stable and effective among all algorithms as the high-probability analysis expected. 

We further compare our algorithm against the other algorithms on the GGRLR task, where the objective function is strongly convex, and experiments results on all the datasets listed in Table~\ref{tab:data}. The experimental results are displayed in 
Figure \ref{fig-GGRLR-complexity}, Figure~\ref{fig-GGRLR-complexity-2}, and Figure~\ref{fig-GGRLR-complexity-3}. Our algorithm still outperforms the other algorithms consistently, which supports our high-probability analysis in the previous sections. We also find that the difference between uniformly averaging and non-uniformly averaging is not significant. One reason is that our algorithm converges within only one or two effective epochs. In this case, non-uniformly averaging will not exhibit its advantage.

\section{Conclusions}
In this paper, we propose a novel convex-concave saddle point formulation to resolve problem~\eqref{prob} as well as the stochastic variant of the PDHG algorithm, named SPDHG. The new algorithm can tackle a variety of real-world problems which cannot be efficeintly solved by the existing stochastic primal-dual algorithms proposed in \cite{Lan-2015-optimal, Zhang-2015-Stochastic, Zhu-2015-Adaptive, Chambolle-2017-Stochastic}. We further provided the high-probability convergence analysis for the proposed SPDHG algorithm when applied to deal with general and strongly convex objectives, respectively.

The proposed SPDHG algorithm is well-suited for addressing compositely regularized minimization problems when the penalty matrix $F$ is non-diagonal. The experiments in performing graph-guided logistic regression and graph-guided regularized logistic regression tasks demonstrated that our algorithm outperforms the other competing stochastic algorithms.

\bibliographystyle{plain}
\bibliography{ref}

\end{document}